\documentclass[letterpaper, 10pt, conference]{ieeeconf}  

\IEEEoverridecommandlockouts                              

\usepackage[english]{babel}
\usepackage{graphicx} 
\usepackage{amsmath} 
\usepackage{amsfonts} 

\usepackage{amsthm} 
\usepackage{multirow}
\usepackage{booktabs}

\DeclareMathOperator*{\E}{\mathbb{E}}
\DeclareMathOperator*{\argmax}{argmax}
\newtheorem{property}{Property}
\newtheorem{corollary}{Corollary}
\newtheorem{proposition}{Proposition}

\title{\LARGE \bf
Optimal Sensing via Multi-armed Bandit Relaxations in Mixed Observability Domains 
}

\author{Mikko Lauri$^{}$ and Risto Ritala$^{}$
\thanks{$^{}$M. Lauri and R. Ritala are with Department of Automation Science and Engineering, Tampere University of Technology, P.O. Box 692, FI-33101 Tampere, Finland. Email: {\tt\small mikko.lauri@tut.fi, risto.ritala@tut.fi}}%
}

\begin{document}

\maketitle
\thispagestyle{empty}
\pagestyle{empty}

\begin{abstract}

Sequential decision making under uncertainty is studied in a mixed observability domain. The goal is to maximize the amount of information obtained on a partially observable stochastic process under constraints imposed by a fully observable internal state. An upper bound for the optimal value function is derived by relaxing constraints. We identify conditions under which the relaxed problem is a multi-armed bandit whose optimal policy is easily computable. The upper bound is applied to prune the search space in the original problem, and the effect on solution quality is assessed via simulation experiments. Empirical results show effective pruning of the search space in a target monitoring domain.

\end{abstract}

\section{INTRODUCTION}

Deploying autonomous agents such as robots equipped with an appropriate set of sensors allows automated execution of various information gathering tasks. The tasks can include monitoring and identification of spatio-temporal processes, automated exploration, or other data collection campaigns in environments where human presence is undesired or infeasible. Robots are mobile sensor platforms whose actions are optimized to maximize the informativeness of measurement data.

As the target state is not known, a probability density function (pdf) over the state, called a belief state, is maintained. Information conveyed by measurement data is incorporated into the belief state by Bayesian filtering. Assuming Markovian dynamics and conditional independence of measurement data given the system state, the problem is a partially observable Markov decision process, or POMDP \cite{Kaelbling1998}.

Optimal information gathering has been studied in the context of sensor management \cite{Hero2007}, and a review of applying POMDPs for sensor management is presented in \cite{Chong2009b}. The problem is formulated as a decision process under uncertainty. The goal is to find a control policy mapping belief states to actions, that when followed maximizes the expected sum of discounted rewards over a horizon of time. The reward associated with an action may depend either on the true state of the system or the belief state. The former can encode objectives such as reaching a favorable state or avoiding costly ones, useful e.g.\ for navigation and obstacle avoidance. The latter option allows information theoretic rewards, such as mutual information, applied in various sequential information gathering problems in robotics, see e.g.\ \cite{Charrow2014,Atanasov2014,Lauri2014}. Indefinite-horizon problems that terminate when a special stopping action is executed are a natural model for tasks that may be stopped once a certain level of confidence about the state is reached \cite{Hansen2007}. 


Finding optimal policies for POMDPs is computationally hard \cite{Papadimitriou1987}, and several approximate methods have been suggested. Point-based algorithms \cite{Spaan2005,Pineau2006} track so-called alpha vectors at a set of points in the belief space. The alpha vectors may then be used to approximate the optimal policy at any belief state. Online planning methods \cite{Ross2008} find an optimal action for the current belief state instead of a representation of the optimal policy. The problem is cast as a search over the tree of belief states reachable from the current belief state under various action-observation histories. Combining online methods with Monte Carlo simulations to evaluate utility of actions has lead to approximate algorithms able to handle problems with up to $10^{52}$ states \cite{Silver2010}.

In mixed observability domains a part of the state space is fully observable. The belief space is a union of low-dimensional subspaces, one for each value of a fully observable state variable. Robotic systems often exhibit mixed observability which may be exploited to derive efficient POMDP algorithms \cite{Ong2010}.

A multi-armed bandit (MAB) is a model for sequential decision-making also applied in sensor management \cite{Hero2007}. A decision-maker plays one arm of the MAB and collects a reward depending on the state of the arm. The arm then randomly transitions to a new state while other arms remain stationary. Solutions to MABs are index policies that are easier to compute than solutions to general POMDPs \cite{Hero2011}.

Most of the aforementioned research applies reward functions that only depend on the true state and action. The expectation of the reward is linear in the belief state, a feature leveraged by many of the solution algorithms. Information theoretic quantities such as entropy and mutual information that would be useful as reward functions in optimal sensing problems are nonlinear in the belief state. Classical POMDP algorithms cannot be applied to solve such problems.

In this paper, we study POMDPs with mixed observability with mutual information as the reward function. As such, our approach is especially suited for optimal sensing problems in robotics domains. We remove constraints on available actions to obtain a relaxed problem. The optimal value of the relaxed problem obtained is an upper bound on the optimal value in the POMDP. We identify the conditions under which the relaxed problem is a MAB and has an easily computable optimal solution. The upper bound is applied in an online planning algorithm to prune the search space.

The paper is organized as follows. In Section~\ref{sec:problem}, the mixed-observability POMDP is defined. In Section~\ref{sec:solve}, methods for solving the problem are discussed. In Section~\ref{sec:bounds}, two relaxations are derived that provide upper bounds for the optimal value function. Section~\ref{sec:mab} determines the conditions under which the relaxations are MABs. Empirical results are provided in Section \ref{sec:experiment}. Section~\ref{sec:conclusion} concludes the paper.

\section{A MIXED OBSERVABILITY POMDP}
\label{sec:problem}
\textit{Notation.} We denote random variables and sets by uppercase letters, and realizations of random variables and members of sets as lowercase letters. Time instants are distinguished by writing e.g.\ $x$ and $x'$ for realizations at time $t$ and $(t+1)$, respectively.

An agent, e.g.\ a robot or another sensor platform, has an internal state $x\in X$ that captures the dynamics and constraints of operating on-board sensors and other devices. The internal state evolves according to a deterministic dynamics model $D_X$, defined $x'=f_a(x)$ where $a\in A(x)$ is a control action in the finite set of actions allowed in internal state $x$.

Let $Y = \lbrace Y_1, Y_2, \ldots, Y_n\rbrace$, $X\cap Y = \emptyset$, denote a set of random inference variables an agent wishes to obtain information about. The dynamics of the variables are governed by a stochastic model $D_Y$, defined as a Markov chain $p(y'\mid y, a)$. The complete state of the system is $s\in S = X \times Y$.


The problem features mixed observability, where the internal state is fully observable and the inference variables are partially observable. The agent's observations $z'\in Z$ follow an observation model $O$, defined by $p(z'\mid y', a)$.

The agent maintains a belief state $b=\left(x,p(y)\right)\in B$, consisting of the deterministic, fully observable internal state and a pdf over $Y$. The initial belief state $b_0$ is given. Given a belief state $b$, an action $a$, and an observation $z'$, the belief state at the next time instant is given by the belief update equation $b'=\tau(b,a,z')=\left(x',p(y'\mid z',a,b)\right)$ where $x'=f_a(x)$, and the pdf over the inference variables is obtained from a Bayesian filter
\begin{equation}
p(y'\mid z',a,b) = \frac{1}{\eta}p(z'\mid y',a) p(y'\mid a, b)
\label{eq:bayesfilter}
\end{equation}
where $p(y'\mid a, b) = \sum\limits_{y\in Y}p(y'\mid y, a)p(y)$ is the predictive pdf and $\eta = p(z'\mid a, b)$ is the normalization factor denoting the prior probability of observing $z'$. Given any sequence of actions and observations, there is no uncertainty about the resulting internal state $x$. Thus we can equivalently define the set of allowed actions $A(x)$ via the belief state as $A_b$.

The agent's objective is encoded by a reward function $R$. The objective is to maximize the expected sum of discounted rewards over a horizon of $T$ decisions. The discount factor is $\gamma\in[0,1]$.

We consider belief-dependent reward functions. Let $R(b,a)=I(Y;Z\mid a)$, i.e.\ the mutual information (MI) between the posterior state and observation. MI is defined
\begin{equation}
I(Y;Z\mid a) = H(Y\mid a) - \E\limits_{Z}\left[ H(Y\mid z', a)\right],
\label{eq:mi_definition}
\end{equation}
where $H(Y\mid a)$ is the entropy of the predictive pdf $p(y'\mid a, b)$ and the second term is the expected entropy of the posterior pdf \eqref{eq:bayesfilter} under the prior pdf $p(z'\mid a, b)$.

The problem $P=\langle S,A,Z,D,O,R,b_0,\gamma \rangle$ where $D=D_X\times D_Y$ is an instance of a POMDP. By Bellman's principle of optimality \cite{Bellman1957} the solution may be found via a backward in time recursion procedure known as value iteration. An optimal value function $V_t^*:B\to\mathbb{R}$ maps a belief state to its maximum expected sum of discounted rewards when an optimal policy is followed for the next $t$ decisions. Optimal value functions are computed by
\begin{subequations}
\begin{align}
Q_t(b,a) &= R(b,a) + \gamma \sum\limits_{z\in Z}{p(z\mid a, b)V_{t-1}^*(b')\mathrm{d}z} \\
V_t^*(b) &= \max_{a\in A_b} Q_t(b,a),
\end{align}
\label{eq:vi}
\end{subequations}
starting from $Q_1(b,a)=R(b,a)$. The optimal policy $\pi_t^*:B\to A_b$ for $t$ remaining decisions is found by extracting the argument $a$ maximizing $Q_t(b,a)$. The recursion is continued up to $V_T$.

\section{SOLVING POMDPS WITH BELIEF-DEPENDENT REWARDS}
\label{sec:solve}
In most POMDPs, the reward function is state-dependent and its expectation is linear in the belief state. The finite-horizon optimal value function then has a finite representation by a convex hull of a set of hyperplanes over the belief space \cite{Smallwood1973}. Many exact \cite{Lovejoy1991} and approximate \cite{Hauskrecht2000,Pineau2006,Spaan2005} offline algorithms for POMDPs rely on this piecewise linearity and convexity of the value function. Reward functions such as mutual information and entropy that are useful in optimal sensing problems are nonlinear in the belief state. Thus, these offline algorithms are not applicable to solve the recursion \eqref{eq:vi} with a belief-dependent reward function.

Online planning methods \cite{Ross2008} find an optimal action for the current belief state instead of a closed form representation of the optimal policy. As explicit representations of policies are not required, a nonlinear belief-dependent reward function does not constitute any additional difficulty.

In online planning, a tree graph of belief states reachable from the current belief state is constructed. The current belief state is the root of the tree, and belief states computed via $\tau(b,a,z')$ are added as child nodes of node $b$. When a desired search depth is reached, the values from the leaves of the tree are propagated back to the root according to \eqref{eq:vi}.

Suboptimal actions may sometimes be pruned from the search tree by branch-and-bound pruning when the optimal value for executing action $a$ in belief state $b$, $Q_t(b,a)$, has an upper bound $U(b,a)$ and a lower bound $L(b,a)$. For a given $a$ and any $\hat{a} \neq a$, if $U(b,a)\leq L(b,\hat{a})$ then action $a$ is suboptimal at $b$ and all its successor nodes may be pruned from the tree. The bounds may similarly be propagated via \eqref{eq:vi}. The number of belief states in the search tree is reduced.

Alternatives to online tree search include e.g.\ specialized approximate methods \cite{Krishnamurthy2007}, however limited to small problems, open-loop approximation applied with the receding horizon control principle \cite{Lauri2014}, or reduced value iteration \cite{Atanasov2014} for Gaussian beliefs over $Y$ in a mixed-observability case. For a theoretical treatment of nonlinear but convex reward functions in POMDPs, we refer the reader to \cite{Araya2010}.

\section{BOUNDS FOR THE VALUE FUNCTION}
\label{sec:bounds}

The optimal policy $\pi_t^*$ attains the optimal value for all belief states. Then any other policy $\pi_t$ achieves a value that is a lower bound on the optimal value. A simple choice is to set $\pi_t$ as the greedy one-step look-ahead policy $\pi_G (b)= \argmax_{a\in A_b}R(b,a)$. Other options include random policies or blind policies \cite{Hauskrecht2000} always executing a single fixed action.


Upper bounds are found by deriving two relaxed versions of the original POMDP problem by removing constraints on the applicable actions. The set of internal states reachable from a subset $X_S \subseteq X$ in a single time step is
\begin{equation}
F(X_S) = \bigcup_{\forall x\in X_S} \bigcup_{\forall a\in A(x)} f_a(x).
\label{eq:reachable_one_step}
\end{equation}
The set of internal states reachable in $k$ steps from $X_S$ is
\begin{equation}
F^k(X_S) = \underbrace{F \circ \ldots \circ F(X_S)}_{\text{$k$ times}}.
\label{eq:reachable_k_step}
\end{equation}
The first relaxation is obtained by removing all constraints imposed by the internal state as follows.

\newtheorem*{unirel}{Universal sensor relaxation}
\begin{unirel}
Given a POMDP problem $P = \langle S,A,Z,D,O,R,b,\gamma \rangle$, its universal sensor relaxation is $P_u = \langle Y,\hat{A},Z,D_Y,O,R,p(y),\gamma \rangle$, where $\hat{A} = \bigcup_{x\in X}A(x)$ contains all actions, $D_Y$ is the stochastic part from $D$, and $b=(x,p(y))$ is replaced by $p(y)$.
\end{unirel}

When we consider only actions applicable in the internal states reachable within $k=T-t$ decisions, where $t$ is the current time step, we obtain the $k$-step sensor relaxation.

\newtheorem*{krel}{$\textbf{\textit{k}}$-step sensor relaxation}
\begin{krel}
Given a POMDP problem $P = \langle S,A,Z,D,O,R,b,\gamma \rangle$, its $k$-step sensor relaxation is $P_k = \langle Y,\hat{A}_k,Z,D_Y,O,R,p(y),\gamma \rangle$, where $\hat{A}_k = \bigcup_{i\in F^k(\lbrace x \rbrace) }A(i)$ is the set of all actions possible in the internal states reachable within $k$ time steps from the current internal state $x$, and $D_Y$ and $p(y)$ are as for $P_u$.
\end{krel}

As $A_b\subseteq\hat{A}_k\subseteq\hat{A}$, the optimal value in either relaxed problem is greater than or equal to the optimal value in the original problem. Let $V^*, V_u^*$, and $V_k^*$ denote the optimal value functions for $P$, $P_u$, and $P_k$, respectively, and let $V^{\pi_G}$ denote the value function for the greedy policy in $P$ for a given $1\leq t\leq T$. Now
\begin{equation}
V^{\pi_G}(b) \leq V^*(b) \leq V_k^*(b) \leq V_u^*(b) \quad \forall b\in B
\label{eq:main_inequality}
\end{equation}
holds for the optimal value and the bounds.

\section{MULTI-ARMED BANDIT INDEX POLICIES FOR POMDP RELAXATIONS}
\label{sec:mab}
Both relaxations defined above are POMDPs themselves. Solving even the relaxed problems may thus be a computationally intractable task. This motivates identifying POMDPs whose relaxations have easily computable optimal policies.

In a multi-armed bandit (MAB) problem, a decision-maker plays one arm of the MAB and collects a reward depending on the state of the arm. Four requirements distinguish MABs among general stochastic control problems \cite{Hero2007}: 1) exactly one machine is played by the agent per action, and the state of that machine evolves such that the agent may not affect it, 2) machines not played remain in their current state, 3) the machines are independent, and 4) the machines that are not played do not contribute any reward. Gittins \cite{Gittins1979} showed that the optimal policies in MABs are so-called greedy index allocation policies. For each arm, an allocation index known as the Gittins index is calculated with the optimal selection yielding the highest index value. Index policies are optimal when actions are not irrevocable \cite{Hero2011,Hero2007}: any action is available at any stage, and may be chosen at a later stage with the same reward, excluding the effect of the discount factor. Index policies are usually much easier to compute than backward induction solutions of POMDPs \cite{Hero2011}.



An index policy is in general not optimal for the mixed-observability POMDP of Section~\ref{sec:problem}, as actions are irrevocable due to the constraints imposed by the internal state. However, both of the relaxations $P_u$ and $P_k$ have a fixed action space. The following three properties are required for the relaxations to be MABs. Results are derived for $P_u$, and they hold for the more restricted case $P_k$ as well.

\begin{property}
Each $a\in \hat{A}$ is related to $Y_a \subset Y$, such that $Y_a\neq \emptyset$ and $i,j\in \hat{A}: i\neq j\Rightarrow Y_i \cap Y_j = \emptyset$. 
\label{prop:separation}
\end{property}

\begin{property}
Given $a\in \hat{A}$, each $y_i' \in Y_a$ is conditional on the values $u_i \in U_i \subseteq Y_a$ of some subset $U_i$ of the inference variables in $Y_a$, and $y_i \notin Y_a$ are stationary, i.e.\
\begin{equation}
p(y'\mid y, a) = \prod\limits_{y_i\in Y_a} p(y_i' \mid u_i)  \prod\limits_{y_i\in Y \backslash Y_a} \delta(y_i'-y_i),
\label{eq:dyn}
\end{equation}
where $\delta$ is the Dirac delta function.
\label{prop:dyn}
\end{property}
\begin{property}
For $a\in\hat{A}$, the observation is conditional on $y_a' \in Y_a$, i.e.\
\begin{equation}
p(z'\mid y', a) = p(z' \mid y_a').
\label{eq:obs}
\end{equation}
\label{prop:obs}
\end{property}

\begin{corollary}
When $p(y) = \prod\limits_{a\in \hat{A}} p(y_a)$, $y_a \in Y_a$, i.e.\ the prior on inference variables is independent between the subsets $Y_a$ and properties \ref{prop:dyn} and \ref{prop:obs} hold, the independence is preserved in the posterior,
\begin{equation}
p(y'\mid z', a, b) = \prod\limits_{k\in \hat{A}} p(y_k'\mid z', a, b).
\label{eq:posterior_separation}
\end{equation}
Furthermore, 
\begin{equation}
I(Y;Z\mid a) = I(Y_a; Z).
\label{eq:mi_separation}
\end{equation}
\label{cor:one}
\end{corollary}
Equation \eqref{eq:posterior_separation} is seen to hold applying \eqref{eq:bayesfilter} to the given prior with models satisfying \eqref{eq:dyn} and \eqref{eq:obs}. Equation \eqref{eq:mi_separation} is seen to hold through two steps. First, due to the independence structure of the prior and posterior, $H(Y\mid a) = \sum_{k\in \hat{A}} H(Y_k\mid a)$ and similarly for $H(Y\mid z', a)$. Second, by \eqref{eq:bayesfilter} we see from \eqref{eq:posterior_separation} that for $k\neq a: p(y_k' \mid z', a, b) = p(y_k) \Rightarrow H(Y_k\mid a) = H(Y_k \mid z', a)$. Applying these steps to \eqref{eq:mi_definition} leads to \eqref{eq:mi_separation}.

We now state our main result determining the conditions under which a POMDP relaxation is a MAB.
\begin{proposition}[MAB equivalence of POMDP relaxations]
When properties \ref{prop:separation}-\ref{prop:obs} are fulfilled, and the prior is $p(y) = \prod\limits_{a\in \hat{A}} p(y_a)$, $y_a \in Y_a$, the relaxations $P_u$ and $P_k$ are multi-armed bandit problems.
\label{prop:mab}
\end{proposition}
\begin{proof}(Sketch).
Consider the four requirements for MABs introduced above. Property \ref{prop:separation} establishes the ''arms'' of the bandit, partly satisfying requirement 1. The rest of requirements 1 and 2 are satisfied by Properties \ref{prop:dyn} and \ref{prop:obs}, which establish the states of the bandit arms as $p(y_a), y_a\in Y_a, a\in \hat{A}$. Requirement 3 is satisfied by the independence properties in the first part of Corollary~\ref{cor:one}. The latter part of the corollary shows that requirement 4 is satisfied.
\end{proof}

When the proposition holds, the optimal policies for $P_u$ and $P_k$ are greedy index policies with values $V_u^* = V_u^{\pi_G}$ and $V_k^* = V_k^{\pi_G}$, respectively. These optimal values are thus much easier to compute than for general POMDPs. 

Let us consider the following example problem. 
\newtheorem*{prdef}{Monitoring reactive targets}

\begin{prdef}
An agent is located at $x\in X=\left\lbrace 1,2,\ldots,M\right\rbrace$. At every time step the agent may either stay where it is or move to one of the neighboring locations $N(x)\subset X$. The applicable actions are $A(x)=x\cup N(x)$. Let $Y=\lbrace Y_1, Y_2, \ldots, Y_M\rbrace$, with $Y_i$ assuming value $1$ if a target is present at location $i$ and 0 if not. Each target $Y_i$ reacts to the agent's presence such that
\begin{equation}
p(y_i'\mid y_i, a) =
\begin{cases}
   w_i(y_i'\mid y_i) & \textrm{if } a = i \\
   r_i(y_i'\mid y_i) & \textrm{if } a \neq i
  \end{cases}
\end{equation}
The agent records measurements in $Z=\lbrace 0, 1\rbrace$ according to
\begin{equation}
p(z' = 0 \mid y', a) = 
\begin{cases}
   1-q_{-} & \textrm{if } y_a' = 0 \\
   q_{+}       & \textrm{if } y_a' = 1
  \end{cases},
\label{eq:example_obsmodel}
\end{equation}
where $q_{-}<0.5 ,q_{+}<0.5$ are the false negative and positive probabilities, respectively, and $p(z' = 1 \mid y', a)= 1- p(z' = 0 \mid y', a)$. The reward function is \eqref{eq:mi_definition}.
\end{prdef}
Consider the relaxations $P_u$ and $P_k$ of this problem. Property~\ref{prop:separation} is immediately seen to be satisfied. Property~\ref{prop:dyn} is satisfied if $r_i(y_i'\mid y_i) = \delta(y_i'-y_i)$ for $1\leq i \leq M$, i.e.\ if targets remain stationary when the agent is not present, while $w_i(y_i'\mid y_i)$ may be chosen freely. Property~\ref{prop:obs} is satisfied as the observation only depends on the value $y_a'$.

\section{EMPIRICAL EVALUATION}
\label{sec:experiment}
We ran simulation experiments on the monitoring problem defined above. There were $|Y|=M=36$ inference variables, arranged on a rectangular two-dimensional four-connected grid. The agent was allowed to move on this grid and sense the targets. We examined two cases. In the first case, all of the properties \ref{prop:separation}-\ref{prop:obs} were satisfied. In the second case, we relaxed Property~\ref{prop:dyn} by allowing all inference variables to change state. In all cases, the optimization horizon $T$ was varied from 1 to 6 decisions. The other parameters were $q_{-}=0.05,q_{+}=0.05, \gamma=0.95$.

We implemented the real-time belief space search (RTBSS) algorithm of \cite{Paquet2006} as presented in \cite{Ross2008}. RTBSS implements an online search of belief states reachable from the current belief state, and applies lower and upper bounds to prune suboptimal actions. We applied the greedy lower bound $V^{\pi_G}$ (Section \ref{sec:bounds}) and upper bounds $V_k^{\pi_G}$ or $V_u^{\pi_G}$ (Section \ref{sec:mab}). We compared this approach to an exhaustive search of all reachable belief states equivalent to using lower and upper bounds $(-\infty,\infty)$, and to the POMCP algorithm \cite{Silver2010}, which gives a recommendation on the next action to execute based on a series of Monte Carlo (MC) simulations.

\subsection{Case 1: Properties \ref{prop:separation}-\ref{prop:obs} satisfied}
\label{subsec:first}
We defined $r_i = \delta(y_i'-y_i)$, and $w_i$ were two-state Markov chains with parameters $p_{01}^i$, $p_{11}^i$, where $p_{jk}^i$ denotes the probability that $Y_i$ transitions from $j$ to $k$. For each $1\leq i\leq M$, we sampled uniformly at random $p_{01}^{i}\in [0.0,0.2],p_{11}^{i}\in[0.8,1.0]$. A set of 1000 initial belief states $(x_0,p(y_0))$ satisfying the independence assumption between inference variables was sampled uniformly at random. 

As Proposition~\ref{prop:mab} holds, $V_u^* = V_u^{\pi_G}$ and $V_k^* = V_k^{\pi_G}$, and the greedy MAB policies give valid upper bounds, see \eqref{eq:main_inequality}. Applying RTBSS with these bounds hence always finds the optimal solution, which was verified in our simulations. The number of visited nodes in the search tree for each of the 1000 belief states is shown in Fig.~\ref{fig:pruning} for $3 \leq T \leq 6$ and both upper bounds. Since the bound $V_k^*$ is tighter, applying it results in a lower or equal number of visited nodes than $V_u^*$. For comparison, the average number of visited nodes for the exhaustive search is shown in Table \ref{tab:exhaustive}. We note that applying either bound greatly reduces the number of visited nodes, in some cases by up to an order of magnitude. Although the reduction in the number of visited nodes is substantial, evaluating the bounds has a computational cost that must be balanced with the savings from visiting fewer nodes. This point is discussed in more detail in the next subsection.

\begin{figure}[thpb]
      \centering
      \includegraphics[scale=0.55]{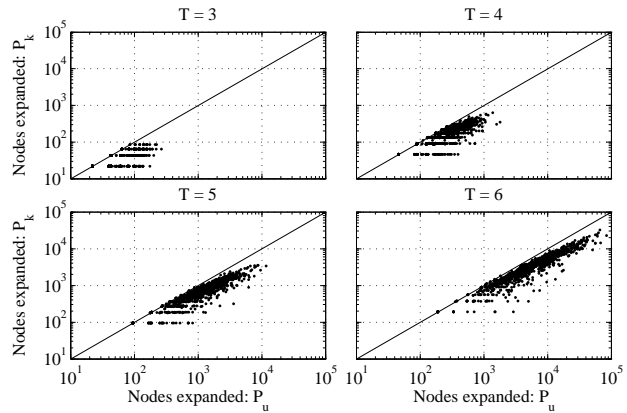}
      \caption{The number of search tree nodes expanded by RTBSS for $3\leq T\leq 6$ with the upper bound from $P_u$ ($x$-axis) or $P_k$ ($y$-axis). The diagonal line shows where the two values are equal.}
      \label{fig:pruning}
\end{figure}

\begin{table}
\vspace{0.3cm}
\begin{center}
\caption{Average number of nodes expanded by exhaustive search.}
\label{tab:exhaustive}
\begin{tabular}{@{}ccccc@{}}
\toprule
      & $T=3$ & $T=4$ & $T=5$ & $T=6$ \\ \midrule
Nodes & $4.0\cdot 10^2$  & $3.8\cdot 10^3$  & $3.6\cdot 10^4$  & $3.5\cdot 10^5$ \\ \bottomrule
\end{tabular}
\end{center}
\end{table}

POMCP recommendations coincide with the optimal action more reliably when the number of MC simulations is increased and the optimization horizon $T$ is short, see Table~\ref{tab:pomcp}. We compared the values of optimal actions to those recommended by POMCP when the two differed. The difference between the two values is the performance loss, for which we computed the mean values and worst-case maximum values. The results are shown in Table~\ref{tab:pomcp_loss}. Performance loss tends to be greater for fewer MC simulations and a greater optimization horizon $T$. As the number of MC simulations increases the mean performance loss is low, indicating that on average POMCP performs very well compared to the optimal solution. However, even if the mean performance loss is low, the worst case performance loss from following POMCP recommendations may be significantly greater. In problems where suboptimal actions may lead to unacceptable performance loss, methods such as RTBSS with valid bounds may be preferable to POMCP.
\begin{table}
\begin{center}
\caption{Percentage of POMCP recommendations agreeing with optimal.}
\label{tab:pomcp}
\begin{tabular}{@{}cccccc@{}}
\toprule
MC simulations & $T=2$ & $T=3$ & $T=4$ & $T=5$ & $T=6$ \\ \midrule
$10^1$       & 40.7    &  36.8   &  31.3   & 28.1    & 29.2    \\
$10^2$       & 57.4    &  50.1   &  45.3   & 43.2    & 40.1    \\
$10^3$       & 69.5    &  62.2   &  54.0   & 47.6    & 46.8    \\
$10^4$       & 75.0    &  74.9   &  69.1   & 63.5    & 59.1    \\ \bottomrule
\end{tabular}
\end{center}
\end{table}

\begin{table*}
\vspace{0.3cm}
\begin{center}
\caption{Performance loss of POMCP compared to optimal.}
\label{tab:pomcp_loss}
\begin{tabular}{@{}ccc|cc|cc|cc|cc@{}}
\toprule
& \multicolumn{2}{c}{$T=2$} & \multicolumn{2}{c}{$T=3$} & \multicolumn{2}{c}{$T=4$} & \multicolumn{2}{c}{$T=5$} & \multicolumn{2}{c}{$T=6$} \\ \midrule
MC simulations & Mean            & Max      & Mean      & Max      & Mean      & Max      & Mean      & Max      & Mean      & Max      \\
$10^1$      & 0.0749          & 0.3930   & 0.0948         & 0.5278         &  0.1061         & 0.5283         & 0.1084    &  0.5587        & 0.1175          & 0.6278         \\
$10^2$      & 0.0298          & 0.2296   & 0.0518         & 0.3299         &  0.0584         & 0.2953         & 0.0642    &  0.3251        & 0.0670          & 0.3968         \\
$10^3$      & 0.0122          & 0.1025   & 0.0294         & 0.1751         &  0.0402         & 0.2035         & 0.0487    &  0.3639        & 0.0518          & 0.2839         \\
$10^4$      & 0.0064          & 0.0503   & 0.0168         & 0.0893         &  0.0261         & 0.1399         & 0.0317    &  0.1816        & 0.0380          & 0.2020         \\ \bottomrule
\end{tabular}
\end{center}
\end{table*}

\subsection{Case 2: Property \ref{prop:dyn} not satisfied}
We next examined the case where Property~\ref{prop:dyn} was not satisfied. We set $r_i = w_i$ for each $i$. Each of the dynamics models was a two-state Markov chain. We considered three subcases distinguished by the rate of the state transitions: slow, medium or fast. For slow dynamics, the parameters were sampled for each $i$ uniformly at random such that $p_{01}^{i,slow}\in [0.0,0.2],p_{11}^{i,slow}\in[0.8,1.0]$, for medium dynamics $p_{01}^{i,med}\in [0.2,0.4],p_{11}^{i,med}\in[0.6,0.8]$, and for fast dynamics $p_{01}^{i,fast}\in [0.4,0.6],p_{11}^{i,fast}\in[0.4,0.6]$. Each experiment was repeated for 1000 randomly sampled initial belief states and dynamics models. All beliefs satisfied the independence assumption between inference variables.

The problem is quite similar to the one in Subsection~\ref{subsec:first}, and POMCP performance was also observed to be very good on average. The MAB equivalence, Proposition~\ref{prop:mab}, is now not satisfied for the relaxed problems. Thus, the upper bounds are approximate, and optimality for RTBSS cannot be guaranteed. We examined the effect that this had on solutions provided by RTBSS. The results are summarized in Table \ref{tab:optimality}. The table shows the percentage of solutions equal to the optimal solution in case of slow, medium or fast dynamics for either the universal sensor upper bound $V_u^{\pi_G}$ from $P_u$ or the $k$-step sensor upper bound $V_k^{\pi_G}$ from $P_k$. 

\begin{table}
\vspace{0.3cm}
\begin{center}
\caption{Percentage of RTBSS solutions agreeing with optimal solution when Property~\ref{prop:dyn} was not satisfied.}
\label{tab:optimality}
\begin{tabular}{@{}ccccccc@{}}
\toprule
Dynamics                    & Bound & $T=2$   & $T=3$    & $T=4$    & $T=5$    & $T=6$    \\ \midrule
\multirow{2}{*}{\textit{Slow}}   & $P_u$   & 100\% & 100\%  & 100\%  & 100\%  & 100\%  \\ 
                        & $P_k$      & 100\% & 99.9\% & 99.9\% & 99.6\% & 99.6\% \\ \midrule
\multirow{2}{*}{\textit{Medium}} & $P_u$   & 100\% & 100\%  & 99.9\% & 100\%  & 100\%  \\ 
                        & $P_k$      & 100\% & 99.8\% & 99.7\% & 99.1\% & 98.8\% \\ \midrule
\multirow{2}{*}{\textit{Fast}}   & $P_u$   & 100\% & 100\%  & 100\%  & 100\%  & 100\%  \\ 
                        & $P_k$      & 100\% & 99.9\% & 99.7\% & 99.0\% & 98.9\% \\ \bottomrule
\end{tabular}
\end{center}
\end{table}

Optimal solutions are found in the majority of cases, with the percentage decreasing as the optimization horizon is greater and the rate of dynamics faster. Since often $V_k^{\pi_G}(b) < V_u^{\pi_G}(b)$, it is more likely that the bound obtained from the universal sensor relaxation does not overestimate the optimal value, and consequently better agreement with the optimal solution is observed. The results suggest that it may still be reasonable to approximate upper bounds for $V^*$ by the value of greedy policies in the relaxed problems $P_u$ or $P_k$, even if their optimality cannot be guaranteed.

Efficiency of pruning the search tree was not affected significantly compared to the case of the previous subsection. Applying either bound dramatically reduced the number of  visited nodes in the search tree. We examined the mean time required to find a solution for a belief state either by exhaustive search or branch-and-bound pruning. A representative comparison is presented in Fig.~\ref{fig:time} for the case of medium dynamics. For $T<3$, exhaustive search performs fastest: the computational burden of computing the bounds outweighs the savings from visiting fewer nodes during the search. The advantages of pruning the search tree become apparent for $T\geq 4$. At best, applying pruning is an order or magnitude faster than exhaustive search. For $T\geq 3$, the upper bound from $P_k$ is fastest. Using the upper bound from $P_u$ is faster than exhaustive search for $T\geq 5$. Comparing computation times between POMCP and RTBSS were not meaningful, as the experiments were run on different computer platforms with different implementations of e.g.\ the search trees.

\begin{figure}[thb]
      \centering
      \includegraphics[scale=0.55]{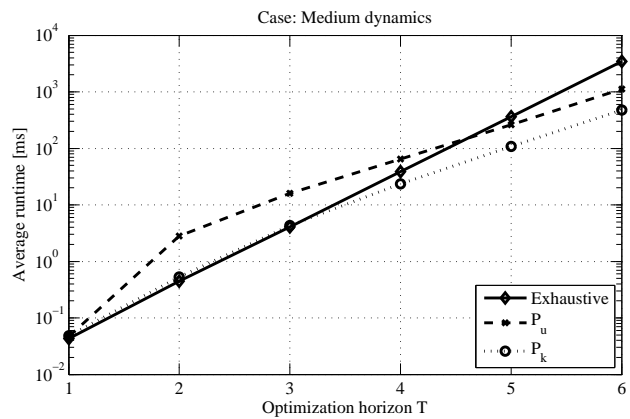}
      \caption{The mean runtime per decision in milliseconds as a function of optimization horizon $T$ for the exhaustive and branch-and-bound search applying upper bounds from $P_u$ or $P_k$.}
      \label{fig:time}
\end{figure}

\section{CONCLUSIONS}
\label{sec:conclusion}
An optimal sensing problem in a mixed-observability domain where an internal state is fully observable and a set of inference variables are partially observable was formulated as a POMDP. The objective was the sequential maximization of mutual information of the inference variables and observations. Upper bounds for the optimal value function were found by relaxing constraints of the original problem.

When three conditions are fulfilled, the relaxed problems are MABs. First, each action is related to a unique subset of inference variables. Secondly, only inference variables in the subset corresponding to the current action evolve, while the other inference variables remain stationary. Finally, observations depend only on the inference variables in the subset related to the current action. The optimal solution of a MAB problem is a greedy index allocation policy, which is much easier to find than solving a general POMDP.

The POMDP was solved by a branch-and-bound search. The effectiveness of the bounds for pruning the search space was empirically verified in a target monitoring problem. Finding an optimal action by requires searching a fraction of the reachable belief states compared to an exhaustive search. Computation time is at best an order of magnitude smaller when applying pruning. The computational savings become apparent when savings due to reduced search space size exceed the additional cost of computing the bounds.


Future work includes studying applicability of our methodology in a wider range of mixed observability domains. Motivated by positive results on optimality of greedy policies for the restless bandit problem \cite{Ahmad2009}, we believe there may exist more classes of stochastic control problems than currently known where a greedy policy is optimal. Identifying such classes would further expand the applicability of our results.

\addtolength{\textheight}{-12cm}   









\bibliographystyle{IEEEtran}
\bibliography{lauri_ritala_ICRA15}

\end{document}